%% file: main.tex
\documentclass{article}
\usepackage[utf8]{inputenc}
\usepackage{amsfonts} % contient mathbb
\usepackage{amsmath} %contient underset
\usepackage{amsthm}
\usepackage{tabularx}
\usepackage{graphicx}
\usepackage{eurosym}
\usepackage{graphics}
\usepackage{tikz}
\usepackage{hyperref}
\usetikzlibrary{shapes,arrows}

\newtheorem{theorem}{Theorem}
\newtheorem{lem}{Lemma}
\newtheorem{rem}{Remark}

\newcommand{\E}{\mathbb{E}}
\newcommand{\Pro}{\mathbb{P}}

\title{Concurrent Neural Network : A model of competition between times series}
\author{Rémy Garnier}
\date{April 2020}

\begin{document}

\maketitle

\begin{abstract}
    Competition between times series often arises in sales prediction, when similar products are on sale on a marketplace. This article provides a model of the presence of cannibalization between times series. This model creates a "competitiveness" function that depends on external features such as price and margin. It also provides a theoretical guaranty on the error of the model under some reasonable conditions, and implement this model using a neural network to compute this competitiveness function. This implementation outperforms other traditional time series methods and classical neural networks for market share prediction on a real-world data set. 
\end{abstract}

\section{Introduction}

Forecasting multiple time series is a useful task, which has many applications in finance \cite{tay2001application} or in supply chain management \cite{alon2001forecasting}. This is however a complex task, because of the number of possible interactions between times series. For instance, standard models such as Vector Autoregression (VAR) \cite{kock2015oracle} generally failed to capture complex behavior in the high-dimensional case.

In this paper, we want to forecast future values of a high-dimensional time series. This time series represents similar assets (or products) that compete with each other. In other words, this means that there is at least a partial cannibalization between times series. Product cannibalization has been defined as “the process by which a new product
gains sales by diverting sales from an existing product” \cite{srinivasan2005identifying}. We want to identify and use cannibalization to improve the demand forecast. Therefore, we need to introduce external covariates (for instance, the prices of different product) that may explain the cannibalization 

We want to apply this model to E-commerce sales data. These data are organized in a hierarchy of categories. For instance, in the family 'HOME', there is a subfamily 'Home Appliance' which contains a category 'Fridge', which can also be further subdivided. It is generally easier to predict aggregated sales for a category than to predict the sales of each product in this category. One of the reasons is the competition between the different products, and the other cross products effects. For instance, the cheapest products and the best-ranked products in the research engine achieve a competitive advantage. However, these advantages do not last forever, with the introduction of new products on the markets. Furthermore, prices and ranking in search engines change every day. Therefore, the competitiveness of each product changes every time step.

In section \ref{sec:Model}, we present the model used to predict E-commerce sales. In section \ref{sec:Theo}, we establish an oracle bound on the estimation risk of our model. In section \ref{sec:Real}, we present the application of our model on the various dataset provided by the French E-Commerce retailer \emph{CDiscount.com}.

\subsection*{Previous Work}

Managers are generally aware of the presence of cannibalization and competition between assets \cite{basuroy2001impact}, but they are few attempts to model and estimate the impact of cannibalization.

 In particular, the very well-known \cite{berry1995automobile} proposes a model of the cross-price elasticity of different products in the U.S. automobile market. They consider the sales of the different models and use household data, such as the composition of the household, its income, its location to model the choice of the consumer. This information is aggregated at the geographic level. This work has been extended by \cite{berry2004differentiated}, which considers individual information on each client, instead of aggregated data. This type of model is often used in the automobile industry to forecast sales. It does however a large amount of data to be used and is generally used only for long-term prevision.

There has also been some work to identify the presence of cannibalization in a different context, for instance in the beverage industry \cite{srinivasan2005identifying}, or in presence of innovative products \cite{van2010estimating}. The interested reader may also refer to this last paper to have a more detailed overview of cannibalization identification. 

The originality of the method proposed in this paper is to use a machine learning approach for modeling competition and to use external covariates to explain cannibalization. We do not use any information on the behavior of the consumer.

\paragraph{Notation.}
We set  $\| x  \| =  \sum_{i=1}^d |x_i|,$ 
if  $x = (x_i)_{i \in [1,d]} \in \mathbb{R}^d.$

\section{Model}
\label{sec:Model}
\input{model.tex}

\section{Estimation Risk bounds on Empirical Risk Estimator}
\label{sec:Theo}
\input{theo.tex}

\section{Inplementation of the model}
\label{sec:Implement}
\input{implement.tex}

\section{Application to E-Commerce sales dataset}
\label{sec:Real}
\input{real_data.tex}
\input{conclusion}

\bibliographystyle{alpha}
\bibliography{biblioConc}

\section{Appendix}
\subsection{Moment bound for Poisson distribution}

\begin{lem}
\label{lem:SkellamBound}
    Let $X$ be a random variable following a Poisson distribution of parameters $\lambda$. Îf we note $M = max
    (1,\lambda e)$, we have the following bound on the moment of X:
    \[\E\big[ |X|^k \big] \leq k! M^{k} \]
\end{lem}

\begin{proof}
The moment generating function of the X is 
\[g(x) = \exp(\lambda * (e^x - 1))\]

We note $m_k = g^{(k)}(0)$ the k-th moment of the distribution. The first derivative of $g$ verify 

\[g'(x) = \lambda \exp(x) g(x)\]

Using Leibniz formula, we have :

\[g^{k+1}(x) =  \lambda \sum_{i=0}^k \dbinom{n}{i} g^{(i)}(x) \exp(x)\]

For all $k$, we have the following recurrent relation

\[m_{k+1} \leq \lambda \sum_{i=0}^k \dbinom{n}{i} m_i \]

We will proove the hypothesis $H_k : m_k \leq k! M^k $ by induction. 
We have $m_0 = 1$ and $m_1 = \lambda$, so $H_0$ and $H_1$ are verified.
For $k>1$, if we suppose $(H_i)$ verified for all $i \leq K $:

\begin{align*}
    m_{k+1} & \leq  \lambda \sum_{i=0}^k \dbinom{k}{i} i! M^i \\
            & \leq  \lambda \sum_{i=0}^k \frac{k!}{(k-i)!} M^i \\
            & \leq  \lambda M^k k! \sum_{i=0}^k \frac{1}{i!} M^{-i} \\
            & \leq  \lambda M^k (k+1)! e^{\frac{1}{M}}
\end{align*}

 As $M \geq 1$ :

\[ m_{k+1} \leq \mu e  M^k (k+1)! \leq M^{k+1} (k+1)!   \]

\end{proof}

\end{document}

%% file: model.tex
\subsection{Observations}

We observe a multi-dimensional time series $X_t = (x_{i,t})_{i \in [1,d]}$ in $\mathbb{N}^d$. For sales prediction $x_{i,t}$ represents the sales of the product $i$ at the date $t$. The sets of products are also supposed to be similar, and in competition with one another. For instance, it could be similar products, or products having the same usage.

We have $n$ observations of this time series. In many cases, $n$ has the same order of magnitude or is smaller than $d$.

We suppose that we know a positive estimator $s(t)$ of $\sum_{i=1}^d x_{i,t}$. Computing such an estimator is generally an easier task, because it is always easier to predict aggregated values than to predict multiple values, and because the global behavior of the series is generally easier to predict than an individual one. Many classical uni-dimensional techniques could be used to compute such an estimator, and we will not discuss this aspect in detail.\cite{box2015time} is a staple for such prediction, but it also possible to cite \cite{kumar2010} for the specific case of sales prediction. 

Let $y_{i,t}=\frac{x_{i,t}}{s(t)}$. It is the \emph{market share} of the product $i$ at the date $t$.

We also observe a covariates series $(\theta_{i,t}) \in \mathbb{R}^p$. These covariates are correlated to the values of the series $x_{i,t}$. In the sales forecasting setting, it could for instance represent the price or the profit margin of the product.

The forthcoming section will lead to a rigorous construction of  the model.

\subsection{Modeling dispersion}
 The first step is to model the dispersion of the series. It is natural to suppose that $x_{i,t}$ are drawn from a Poisson distribution of parameter $\lambda_{i,t}$. 
 
 More precisely, we suppose that it exists $(\epsilon_{i,t})$ independent Poisson process of parameter 1 such that $x_{i,t} = \epsilon_{i,t}(\lambda_{i,t})$. We need to introduce such process $(\epsilon_{i,t})$ in order to distinguish the "natural" stochastic dispersion of the random variable and the variation of the parameter $\lambda_{i,t}$. This distinction would be useful later, because, we hope to explain the variation of the parameter $\lambda_{i,t}$.

In the case of E-commerce sales forecasting, the choice of Poisson distribution has already been done \cite{Ban19}. It has several advantages. 

First, we observed that the sales time series are strongly heteroscedastic and that the local variance of the series is strongly correlated with the local mean of the time series. This phenomenon is 

Second, it allows us to limit the effects of the presence of outliers in our data. Indeed, higher values are more likely than with a Gaussian white noise modeling for instance.

Third, the positive integer values are naturally modeled by counting process. We can suppose, that for each week $t$ and each product $i$, the arrival of clients follows a Poisson process and that the parameter of this Process change for each time period $t$. It implies, that, conditionally to the parameter $\lambda_{i,t}$ is known,  arrival time of the different clients are independent.

\subsection{Modeling competition}

Now, we want to model the competition and cannibalization between the different time series. The main idea is to introduce a weight for every product and every date. This weight represents the competitiveness of each product, which may vary over time. Then, we distribute the sales proportionally to this weight.

More formally, for each series $i$ at each date $t$, we introduce a weight  $w_{i,t}$. Parameters $\lambda_{i,t}$ of the previous section are defined as:

\[ \lambda_{i,t} =  s(t)\cdot  \frac{w_{i,t}}{\displaystyle  1+ \sum_{j=1}^d w_{j,t}}\]

The "+1" is here to ensure that the magnitude of the weight remain the same for all the observed period. Therefore : 

\[ \E[\sum_{i=0}^d x_{i,t}]= s(t)\cdot  \frac{\displaystyle \sum_{i=0}^d w_{i,t}}{\displaystyle  1+ \sum_{i=1}^d w_{i,t}}  \]

If the sum of the weight is large enough, $s(t)$ is a good estimator of the sum $\sum_{i=0}^d x_{i,t}$. 

It is easy to add a new product in this setting just by adding a new weight. It is useful, because of the short sales cycle of numerous products. 

\subsection{Modeling temporal evolution}

In this section, we explain how the weight of the previous section are computed and how they vary time. We suppose  that there is a function $\phi$, such that
\begin{equation}
    w_{i,t} = \phi(y_{i,t-1},\theta_{i,t})
\end{equation}. 
Let us explain the assumption behind this relation. To begin, we should note that the function $\phi$ is applied to two different parameters. The first one relies on the past values of the market share. This is an important value because many intrinsic aspects of the product are coded within the past values. For instance, its quality, its notoriety, its position on the market are reflected in the past sales and they do not change rapidly. In practice, we would consider more than 1 values in the past, which mean that we would have:  $ w_{i,t} = \phi(y_{i,t-k},\dots,y_{i,t-1},\theta_{i,t})$.

The second parameter is the vector of covariates $(\theta_{i,t})$. These covariates should explain the variation of competitiveness. 

Finally, let's remark that we consider that the underlying behavior of each series is the same. Indeed, we use a unique function $\phi$ for all the series we observe instead of using a specific $\phi_i$ for every. It means that the different series are interchangeable and have the same behavior. It allows them to share information between series and to adapt to newly introduce the product on the market.

However, this last assumption has some drawbacks. In particular, when some unknown or unrecorded features are relevant for the prediction of the sales, this could lead to changes in the market share that are not fully explained by this model.

\subsection{Summary}

The model may thus be written as:
\[
\left\{
\begin{array}{rlc}
        X_1 & \sim \mathcal{P}_{X_1} & \\
     x_{i,t} & = \epsilon_{i,t}(\lambda_{i,t}) & \text{for } t > 1 \\
     \lambda_{i,t} & =\displaystyle s(t) \frac{\phi(x_{i,t-1}/s(t-1), \theta_{i,t})}{\displaystyle 1+\sum_{j \in [1,d]} \phi(x_{j,t-1}/s(t-1),\theta_{j,t})} &   
\end{array}
\right.
\]
where
$\mathcal{P}_{X_1}$ denotes the distribution of the first values. With this model $(X_t)$ is a (non-homogeneous) Markov chain with a transition function $F_t$ such that 

\[X_t = F_t(X_{t-1}, \epsilon_t)\] 

where $\epsilon_{t} = (\epsilon_{1,t}, \dots \epsilon_{i,t})$ . More precisely

\begin{equation} \label{eq:resume}
    F_t(X, \epsilon)  = \epsilon_t             \left(s(t)\frac{\phi(\frac{X}{s(t-1)},\theta_{i,t})}{1+\|\phi( \frac{X}{s(t-1)},\theta_{t}) \|}\right) 
\end{equation} 

This models can be viewed as an extension of multivariate count auto-regressive model (\cite{Fokianos_2020}), where we add non-linear relation between times series.

%% file: theo.tex
In this section, we  establish theoretical bounds on the estimation risk of our model. Contrary to \cite{doukhan2012weak}, we cannot use weak dependence hypotheses. Instead we are using the exponential inequality introduced by Dedecker and Fan in \cite{dedecker2015deviation} and extended for the non-stationary times series in  \cite{alquier2019exponential}. 

\subsection{Contraction condition}

In order to apply the result of  \cite{alquier2019exponential}, a contraction condition on the Markov transition function must be verified. More precisely, there must be a constant $\rho \in [0,1[$ such that for all $X,X' \in \mathbb{R}^d\,$:

\begin{equation}
    \underset{t}{\text{sup }} \E \Big[  \| F_t(X, \epsilon) - F_t(X', \epsilon)\| \Big] \leq \rho \|X-X'\|
    \label{eqn:c1}
\end{equation}{}

This condition is verified on some condition on the function $\phi$

\begin{lem}

Assume that we have a weight function $\phi$ and a seasonality $s$ defining a transition function $F_t$ as \eqref{eq:resume}. If $\phi$ and $s$ are such that: 

\begin{enumerate}
    \item There is a constant $\tau_s$  such that, for all $t$:
    \[\frac{s(t+1)}{s(t)} \leq \tau_s \],
    \item There is a constant $\tau \in \mathbb{R}^+$ such that  for all $x \in \mathbb{R}^+, \theta \in \mathbb{R}^p$ 
    \[\frac{\delta \phi}{\delta x}(x,\theta) \leq \tau\],
    \item 
    $ 3 \tau_s \tau \leq 1  $.
%    \item There is a constant $\rho < 1$, such that for all $x \in \mathbb{R}^+, \theta \in \mathbb{R}^p$ :
%    \[ \phi(x, \theta) \geq \frac{2 \tau \tau_s}{\rho d} \]
\end{enumerate}
Then the random iterated system $F_t$ fits the contraction condition \eqref{eqn:c1} for any $\rho<1$ such that  $ 3 \tau_s \tau \leq \rho$.
\end{lem}

\begin{proof}

Note $G_t(X)= \frac{\phi(\frac{X}{s(t-1)},\theta_{i,t})}{1+\|\phi( \frac{X}{s(t-1)},\theta_{i})\|}$
 
For $X, X' \in \mathbb{R^+}^d$

\begin{align*}
    \E \Big[  \| F_t(X, \epsilon) - F_t(X', \epsilon)\| \Big] &\leq& &  s(t) \| G_t(X) - G_t(X')\|   \\
     &= & & s(t) \frac{\| (1 +\|\phi( \frac{X'}{s(t-1)},\theta)\|) \phi( \frac{X}{s(t-1)},\theta)- (1+\|\phi( \frac{X}{s(t-1)},\theta)\|)\phi( \frac{X'}{s(t-1)},\theta) \|}{(1+\|\phi( \frac{X}{s(t-1)},\theta)\|) \cdot (1+\|\phi( \frac{X'}{s(t-1)},\theta)\|)}  \\
    &\leq & & s(t) \|\phi( \frac{X}{s(t-1)},\theta)-\phi( \frac{X'}{s(t-1)},\theta)\|  \\
    & & & + s(t)  \frac{| \|\phi( \frac{X'}{s(t-1)},\theta)\|- \|\phi( \frac{X}{s(t-1)},\theta) \| |}{1 + \|\phi( \frac{X}{s(t-1)},\theta)\|} \\
    & & & + s(t) \frac{ \|  \phi( \frac{X}{s(t-1)},\theta)- \phi( \frac{X'}{s(t-1)},\theta) \|}{1+\|\phi( \frac{X}{s(t-1)},\theta)\|}\\
    & \leq & & s(t) \|\phi( \frac{X}{s(t-1)},\theta)-\phi( \frac{X'}{s(t-1)},\theta)\| + 2 s(t) \frac{ \|  \phi( \frac{X}{s(t-1)},\theta)- \phi( \frac{X'}{s(t-1)},\theta) \|}{1+ \|\phi( \frac{X'}{s(t-1)},\theta)\|}
\end{align*}

Using the condition 2, $x \longrightarrow \phi(x,\theta)$ is $\tau$-Lipschitz for all $\theta$. Therefore, we have :

\[\E \Big[  \| F_t(X, \epsilon) - F_t(X', \epsilon)\| \Big] \leq 3 \frac{s(t)}{s(t-1)} \tau  \|X - X'\|  \]

Using condition 1 and 3, we have :

\[\E \Big[  \| F_t(X, \epsilon) - F_t(X', \epsilon)\| \Big] \leq \rho \|X - X'\|  \]

which concludes the proof
\end{proof}

Let us discuss the condition on lemma 1 :

The first and the third one put a condition on the regularity of the seasonality, which should not change too abruptly. In practice, this is not always verified. Indeed, some event like Black Friday creates a drastic changes in product sales seasonality. Ideally, such event should be retreated or handled with other techniques. Otherwise, changes in seasonality are mostly smooth.

The second and third one put a condition on the variation of the weight function. This puts a constraint on the type of model we use to build this weight function, which should be smooth. In particular, tree-bases models such doesn't satisfy this condition. However, for neural network it is similar to the stability condition described in \cite{miller2019stable}, where the author  show that enforcing this conditions does not degrade the performance of recurrent neural network.

To compute a generalisation bound on our model we need to introduce:

\begin{align*}
    G_{X_1}(x) & = \int \| x-x'\| dP_{X_1}(dx') \\
    H_{t, \epsilon}(x,y) &= \int \| F_t(x,y) - F_t(x,y') \| dP_{\epsilon}(dy')
\end{align*}

To have a Bernstein inequality, we need some constraints on the dispersion of our the times series. More precisely, we need to have  the following inequalities for some constants $M >0$, $V_1 > 0$ and $V_2 > 0$ such that, for all integer $k \geq 2$  :

\begin{align}
   \E[G_{X_1}(x)^k] & \leq \frac{k!}{2} V_1 M^{k-2}
   \label{eq:3}\\
   \E[H_{t, \epsilon}(x,\epsilon)^k] & \leq \frac{k!}{2} V_2 M^{k-2}
   \label{eq:4}
\end{align}

\begin{lem}
If we have $R$ such that, for all t, $s(t) \leq R$, then if we note $M = d \max(1, e R)$ and $V_1 = 4 M^2 = 4  d^2 max(1, e R)^2$, it holds :

\[ \E[H_{t, \epsilon}(x,\epsilon)^k] \leq \frac{k!}{2} V_2 M^{k-2} \]
\end{lem}

\begin{proof}
Let's consider 
\[ \E[H_{t, \epsilon}(x,\epsilon)^k] = \int ( \int \| F_t(x,y) - F_t(x,y') \| dP_{\epsilon}(dy') )^k dP_{\epsilon}(dy)\]

Using Jensen inequality :

\begin{align*}
    \E[H_{t, \epsilon}(x,\epsilon)^k]  & \leq \int  \int  \| y(s(t)  G_t(x))  -  y'(s(t)  G_t(x) ) \|^k  dP_{\epsilon}(dy') dP_{\epsilon}(dy)  \\
                                       & \leq \E[ \| Y - Y' \|^k ]
\end{align*}

Where $Y = (Y_i)$ and $Y' = (Y_i')$ are independant vector of independant random variables following a Poisson distribution of parameters $s(t)  G_t(x)$. It holds :

\[\E[H_{t, \epsilon}(x,\epsilon)^k] \leq \E[  \|Y\|^k + \|Y'\|^k ] \leq 2 \E[  \|Y\|^k ]\]

As $\| G_t(X) \|_{\infty} \leq 1$, we have 

\begin{align*}
  \E[  \|Y\|^k ]  & \leq d^k \E[  \|Y\|_{\infty} ^k ] \\
    & \leq  d^k \E[(y_t)^k]
\end{align*}

Where $y_t$ is a random variable following a poisson process of parameter $s(t) $

Using lemma \ref{lem:SkellamBound}, we have $\E[(y_t)^k] \leq k! \max(e s(t),1)^k$. This ensures

\[\E[H_{t, \epsilon}(x,\epsilon)^k] \leq 2 d^k k! \max(e s(t),1)^k \]

and allow us to concludes.
\end{proof}

\subsection{Risk Bounds on Empirical Risk Estimator}

In this section, a bound on model selection error is provided. 

Let $(X_t)$ be an $\mathbb{R}^d$ valued process with $n$ observations following the model described in part $1$ for a function $\phi^*$ .Let  $S$ be a set of functions respecting the condition of the Lemma 1 such that $\phi^* \in S$. For a function $\phi \in S$, we define an empirical risk :

\[R_n(\phi) = \frac{1}{n} \sum_{t=2}^n \| X_{t+1} -   s(t+1)\frac{\phi(\frac{X_t}{s(t)}, \theta_t)}{ 1+ \| \phi( \frac{{X_t}}{s(t)}, \theta_t ) \| } \|  \]

We also define:

\[ R(\phi) = \E [R_n(\phi)] \]

We define the minimum empirical risk estimator:

\[ \widehat{\phi} = \underset{\phi \in S}{\text{ argmin }} R_n(\phi)  \]

It is possible to bound the estimation risk :

\begin{theorem}
Let $K_t(\rho) = \frac{1 - \rho^t}{1- \rho}$. If $\phi$ and $X = (X_i)$ verified the condition (1) to (4), then for $\delta > 0$ we have with probability $1- \delta$:

\[  R(\widehat{\phi}) \leq  R(\phi^*) + (1+\tau) \left( \frac{  \sqrt{ 2V_2 \log(\frac{1}{\delta})}}{\sqrt{n}} + \frac{\sqrt{ 2V_1 \log(\frac{1}{\delta})}}{n} + \frac{  2MK_{n-1}(\rho)\log(\frac{1}{\delta})}{n} \right) \]

\end{theorem}

\begin{rem}
We observe the usual decay in $\mathcal{O}(\sqrt{ \frac{\log(\frac{1}{\delta})}{n}})$. If we use the values for $V_1$ establish in the lemma 2, we observe that the error grows linearly with the dimension $d$.
\end{rem}

\begin{proof}

    First, let's recall the usual argument to  bound the excess risk : 
    
    \begin{align*}
        R(\widehat{\phi}) - R(\phi^*) & = R(\widehat{\phi}) - R_n(\widehat{\phi}) + R_n(\widehat{\phi}) - R_n(\phi^*) + R_n(\phi^*) - R(\phi^*) & \\
        & \leq |R_n(\phi^*) - R(\phi^*)| + |R(\widehat{\phi}) - R_n(\widehat{\phi})|  & \text{ (by definition of } \widehat{\phi})\\
    \end{align*}

Therefore for all $t > 0$, it holds :

\begin{align*}
    \Pro[R(\widehat{\phi}) - R(\phi^*)  \geq t] & \leq \Pro[|R_n(\phi^*) - R(\phi^*)| + |R(\widehat{\phi}) - R_n(\widehat{\phi})|  \geq t] \\
    & \leq \Pro[|R_n(\phi^*) - R(\phi^*)|\geq \frac{t}{2}] + \Pro[|R(\widehat{\phi}) - R_n(\widehat{\phi})| \geq \frac{t}{2}] \\
\end{align*}

Thus:

\begin{equation}\label{ineq:classic}
    \Pro[R(\widehat{\phi}) - R(\phi^*)  \geq t]\leq 2 \underset{\phi \in S}{\text{ sup }} \Pro[|R(\phi) - R_n(\phi)| \geq \frac{t}{2}]
\end{equation}{}

    Then, we aim at bounding the difference $R_n(\phi) - R(\phi)$ for all possible functions $\phi$.

    Let $(\mathcal{F}_k)$ be the natural filtration of the chain $(X_k)$
    %Let's consider $g_{k,\phi}(X_1,\dots, X_k) = \E( R_n(\phi) |\mathcal{F}_k)$ and $d_k = g_{k,\phi}(X_1,\dots, X_k) -g_{k-1,\phi}(X_1,\dots, X_{k-1})  $. 
    
    %\[  R_n(\phi) - R(\phi) = \sum_{k=2}^n d_k \]
    
    $R_n$ is $\frac{(1+\tau)}{n}$ Lipschitz separable. Therefore for $\epsilon > 0$ , we can apply the theorem 3.1 of \cite{alquier2019exponential} to $\frac{R_n}{(1+\tau)}$. Actually, we use a slightly different version, as the space of Poisson processes are not actually separable. However, being able to bound $\E [H_{t, \epsilon}(x,y)^k ] $ suffice to use their version of Bernstein inequality.
    
    \[ \Pro[| R_n(\phi) - R(\phi)| \geq \frac{(1+\tau)}{n} \epsilon] \leq \exp\left(\frac{-\epsilon^2}{2 V_1 + 2(n-1)V_2 + \epsilon MK_{n-1}(\rho)}\right)\]
    
    Hence, we have with probability at least $1 - \delta$ :
    
    \[| R_n(\phi) - R(\phi)| \leq \frac{(1+\tau) \sqrt{(2 V_1 + 2(n-1)V_2 )\log(\frac{1}{\delta})}}{2n} + \frac{(1+\tau)  MK_{n-1}(\rho)\log(\frac{1}{\delta})}{n} \]
    
    Therefore, using (\ref{ineq:classic}), with probability at $1- \delta$, we have:
    
    \[  R(\widehat{\phi}) \leq  R(\phi^*) + (1+\tau) \left( \frac{  \sqrt{ 2V_2 \log(\frac{2}{\delta})}}{\sqrt{n}} + \frac{\sqrt{ 2V_1 \log(\frac{2}{\delta})}}{n} + \frac{  2MK_{n-1}(\rho)\log(\frac{2}{\delta})}{n} \right) \]
    
\end{proof}

%% file: implement.tex
Now we present the implementation of the theoretical model model proposed in section \ref{sec:Model} and how it could be used for times series prediction. The code is available on \href{https://github.com/garnier94/Concurrent_Neural_Network}{Github} \cite{Git}.

\subsection{Empirical risk minimization}

We want to adapt our model to a classical machine learning setting, using empirical risk minimization, in order to be able to use efficient optimization algorithm. To do so, we introduce a set $\Phi$ of possible weight function $\phi$. We will search for the optimal function in this set. 

We will perform a prediction at an horizon $h \geq 1$. For a choosen weight function $\phi$ a covariate vector $\theta_{i,t}$ and known past value of the market share $y_{i,t-h}$, the next value will be predicted as :

\[\hat{y}_{\phi,i,t} =\frac{\phi(y_{i,t-h}, \theta_{i,t})}{\displaystyle 1+\sum_{j \in [1,d]} \phi(y_{j,t-h},\theta_{j,t})}\]

We will then perform the empirical risk minimization for a loss $L$: 

\begin{equation}\label{eq:empirical_loss}
    \hat{\phi} =\underset{\phi \in \Phi}{\text{ argmin }} \sum_{t=1}^{n-h} \sum_{i=1}^d L(y_{i,t}, \hat{y}_{\phi,i,t}) 
\end{equation}

Note that, when $L$ is the Poisson Loss function $L(y,\hat{y}) = \hat{y} - x \log\hat{y}$, the empirical risk minimizer is also the function $\phi$ which minimize the log likelihood of the model presented in section \ref{sec:Model}. However, we will also use the more standard $L_1$ Loss function  to show the interest of the Poisson distribution.

\subsection{Concurrent Neural network}

The most complex choice is the choice of the set $\Phi$ among which we choose the weight function $\phi$. It should be able to satisfy several properties.

First, it should be complex enough in order to handle non linear behavior. Indeed, we want to model complex behavior, that depends interacting and sometimes correlated parameters. 

Second, the considered functions must be differentiable. This is a condition necessary to use powerful optimization algorithm. This condition disqualify most of tree-based models, often used to predict sales in different context (citation needed). 

\input{tikz_2.tex}

This constraint leads us to use feed-forward neural network as sub-models $\phi$. Around them we build a structure that we called a concurrent neural network model that we will note Conc-NN to distinguish it from Convolution Neural Network traditionally abbreviated CNN. We summarize this approach on figure \ref{fig:Conc_NN}.

Note that we introduce a scale factor $\alpha < 1$ . This is because the sum $\sum_{i=1}^{d} y_{i,t-h}$ may be smaller than 1 in some case because of the presence of newly introduced product between the date where the prediction is made $t-h$ and the date where the prediction is actualized $t$. 

In order to stay simple and not introduce any bias in the comparison between models, we do not choose a data-driven values for $\alpha$. For short and medium-term horizon, $\alpha=1$, but for long term horizon, we consider $\alpha=0.8$.  

\subsection{Neural network architecture and training}

Four neural network estimator are presented in the results.
\begin{itemize}
    \item \textbf{FF-NN} Classical Feed-forward neural network trained with L1 Loss
    \item \textbf{L1-Conc-NN} Concurrent neural network presented in the last subsection trained with L1-Loss
    \item \textbf{P-Conc-NN} Same model trained with Poisson Loss
    \item \textbf{L1-Pre-Conc-NN} Concurrent neural network  trained with L1-Loss using the pretrained weight of FF-NN
\end{itemize}

The training method are the same for all this models. We use simple feed-forward architecture with less than four layers, and less than 32 neurons per layers. We use a RELU activation function for every layer except the last one where we use a SoftPlus activation. For every category of products, we use a validation period to perform model selection among 10 different neural network architectures. 
Data are introduced by batch in a random order. Every batch correspond to a week, to allows an easy rescaling in the case of Concurrent Neural Network.  

%% file: tikz_2.tex
 \tikzstyle{block} = [draw, fill=blue!20, rectangle, 
    minimum height=3em, minimum width=6em]
\tikzstyle{sum} = [draw, fill=blue!20, circle, node distance=1cm]
\tikzstyle{input} = []
\tikzstyle{output} = []
\tikzstyle{pinstyle} = [pin edge={to-,thin,black}]

 \begin{figure}
     \centering
     \begin{tikzpicture}[auto, node distance=2.2cm,>=latex']
        % Block
        \node [input, name=input_1] {{$y_{1,t-h}$}};
        \node [block, right of=input_1, pin={[pinstyle]above:$\theta_{1,t}$},
                node distance=3cm] (system_1) {NN  $\phi$};
        \node [output, right of=system_1,node distance=6.5cm] (output_1) {$\widehat{y}_{1,t-h}$};
        \draw [draw,->] (input_1) -- node {} (system_1);
        
        %Dots 
        \node [below of=system_1, name=dot ,node distance=2cm]{\dots};
        \node [block,right of=dot,node distance = 3.5cm ] (Frac) {$\alpha\frac{\cdot}{\displaystyle 1+ \sum_{i} \cdot}$};
        \node [below of=output_1, name=dot_2 ,node distance=2cm]{\dots};
        % Block 
        \node [block, below of=dot, pin={[pinstyle]above:$\theta_{d,t}$}] (system_2) {NN  $\phi$};
        \node [input, name=input_2, left of=system_2, node distance = 3cm] {{$y_{d,t}$}};
        \node [output, right of=system_2,node distance=6.5cm] (output_2) {$\widehat{y}_{d,t}$};
        \draw [draw,->] (input_2) -- node {} (system_2);

        \draw [->] (system_1) -- node [above,name=y] {$w_{1,t}$} (Frac);
        \draw [->] (system_2) -- node [below,name=y] {$w_{d,t}$}(Frac);
        \draw [->] (Frac) -- node [name=y] {}(output_1);
        \draw [->] (Frac) -- node [name=y] {}(output_2);
    \end{tikzpicture}
     \caption{Concurrent NN Model}
     \label{fig:Conc_NN}
 \end{figure}
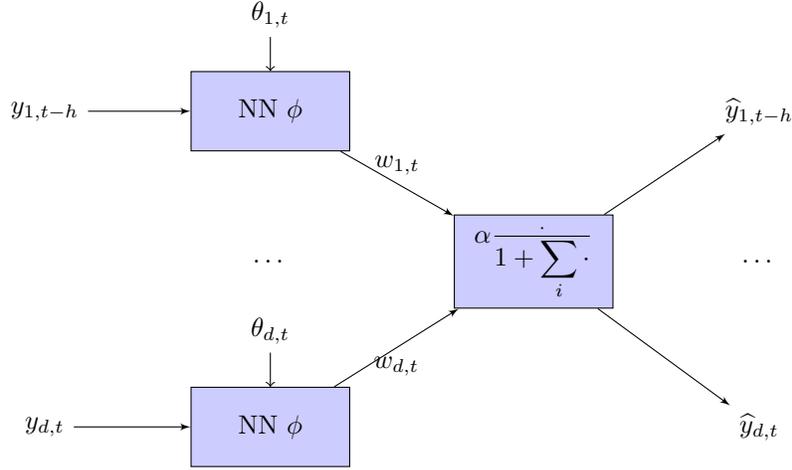{}

%% file: real_data.tex
In this section, we will try to apply our method to a real dataset of E-commerce sales. First, we present the data used in our application

\subsection{Datasets}

We consider different data sets coming from the E-commerce company \emph{Cdiscount}. It is a French E-commerce retailer selling a large variety of products. 

We use the data available for different families of products sold by CDiscount. These categories have been selected to represent various types of products. The hyper-parameters of the models were chosen using other families. The product categories are presented in Table \ref{tab:data}, along with the number of products $d$ and some descriptive statistics.SD stands for Standard Deviation, where NSD (Normalized Standard Deviation) is the ratio between SD and the Weekly sales average by product. Average and standard deviation are computed only when the products are actually proposed on the website. 

These datasets can be roughly separated into three categories. The first one is the product that presents regular seasonality and where the demand is relatively insensitive to price variation. It is the case for Baby chairs and Freezers, which have a small NSD. The second one is products that present strong seasonality factors, such as Lawn Mowers and Scooters, but are not highly sensitive to price changes. The last ones are products such as Smartphones and TVs, which present short sales cycles and/or are very sensitive to price changes, which translates into an important NSD.

\begin{table}
    \centering
    \resizebox{\columnwidth}{!}{
    \begin{tabularx}{15cm}{p{6cm}|X|X|X|X|X|X}
         Family & $d$ & Total weekly sales average  &  Weekly sales average by product & Max sales & SD & NSD \\
         \hline
      Baby chair & 127 & 483.0 & 9.8 & 238 & 16.3 & 1.66 \\
      Freezer & 139 & 1364.0 & 23.9 & 567 & 39.3 & 1.64 \\
      Keyboard &  68 & 267.8 & 9.3  & 374  & 19.2 & 2.06\\
      Lawn Mower  & 81 & 369.3 & 13.5 & 455 & 27.4 & 2.02 \\
      Scooter  &  589  & 1927.1  & 8.8 & 785 & 20.3 & 2.30 \\
      SD Cars  & 45 & 288.6  & 16.4 & 542  & 36.1 & 2.2 \\
      Smartphone  & 1055 & 8352.3 & 29.8 & 2886 & 81.6 & 2.74 \\
      TVs & 535 & 8004.9 & 64.8 & 5547 & 148.2 & 2.29\\
    \end{tabularx}
    }
    \caption{Datasets}
    \label{tab:data}
\end{table}

We consider the weekly sales starting from January 2017 to December 2020. The first three years are used to train the models, which are evaluated on the last year of data. Added external features include the margin practiced on the product and their prices.

\subsection{Evaluation}

We want to predict the weekly sales shares of different products for an horizon of $h$ weeks. We will evaluate the prediction using the usual Mean Absolute Percentage Error(MAPE) . For a prediction $\widehat{y}_{i,t+h}$, it is defined as:

\[MAPE =  100 * \frac{\sum_{i=0}^d \sum_{t=0}^{T-h} |\widehat{y}_{i,t}-y_{i,t}|}{\sum_{i=0}^d \sum_{t=0}^{T-h} y_{i,t}} \]

The MAPE has two main advantages. First, it can deals with outlier, which tends to be over-weighted with other metrics such as RMSE and NMRSE. First, it scales with the level of sales, and so can be used to compare the prediction for the different product categories. However, MAPE error tends to favor under-estimated predictions. 

We will use other predictors to get a benchmark of prediction.

\begin{itemize}
    \item \textbf{Last Value (LV) :} Use the last known value to predict future market share. 
    \item \textbf{Moving average (MA) :} Use a moving average model to predict the future values. Hyper-parameters are calibrated on a validation period.
    \item \textbf{Random forest  (RF) : } Random forest using the same features as ConcNN. Different architectures are cross-validated on a validation period.
    \item \textbf{Scaled Random forest (S-RF):} We also use Random Forest where the prediction are scaled to match the total number of sales. This is useful to compare our models with the results when we perform a simplistic re-scaling after a model is trained.
\end{itemize}

\subsection{Results}

We  present the results in the Table \ref{tab:Result_1} for horizon $h=4$ (1 month ahead),in the Table \ref{tab:Result_2} for horizon $h=8$ (2 months ahead), and in the Table \ref{tab:Result_3} for horizon $h=12$ (3 months ahead). In some cases, FF-NN were not able to produce any meaningful prevision and only predict $0$. In this case, we put a star (*) in the columns. The best model for every set of products and every horizon is in $\textbf{bold}$.

L1-Pre-Conc-NN also failed to produce any prevision other than 0. It is of course the case when FF-NN provides zero predictions, but it can also happen when pretrained weights are too small. We also denote this case by a (*).

\begin{table}[]
    \centering
\resizebox{\columnwidth}{!}{
\begin{tabularx}{19cm}{p{3cm}|p{1cm}|X|X|X|X|X|X|X|X|X}
    \textbf{Category} & LV & MA & RF & S-RF & FF-NN & L1-Conc-NN & P-Conc-NN & L1-Pre-Conc-NN  \\
      \hline
      Baby chair & 76.7 & 73.8 & 70.7 & 70.5 & 67.7 & \textbf{59.9} & 63.5 & * \\
      Freezer & 88.1 & 85.3  & 69.7 & 70.1 & 67.1 & 66.3 & \textbf{65.2} & 68.3  \\
      Keyboard & 87.6 & 81.7 & 76.7 & 77.8  & \textbf{67.1} & 72.7 & 70.3  & * \\
      Lawn Mower & 83.2 & 81.5 & 72.8 & 75.7 & 74.6 & 75.6 & \textbf{72.5} & 74.1 \\
      Scooter & 86.3 & 84.1 & 78.5 & 82.0 & 75.2 & \textbf{73.6} & 75.7 & 74.0 \\
      SD Cars & 83.7 & 79.3 & 88.9 & 90.3 & \textbf{74.0} & 75.8 & 77.9 & * \\
      Smartphone & 84.0 & 81.6 & 79.1 & 84.3 & 81.3 & \textbf{75.6} & 75.8 & 75.7  \\
      TVs & 78.5 & 80.7 & 76.2 & 78.2 & * & \textbf{71.7} & 77.8 & * \\
     \hline
\end{tabularx}
}
    \caption{MAPE Results on the market share prediction for an short-term horizon $h=4$ weeks}
    \label{tab:Result_1}
\end{table}

\begin{table}[]
    \centering
\resizebox{\columnwidth}{!}{
\begin{tabularx}{19cm}{p{3cm}|p{1cm}|X|X|X|X|X|X|X|X|X}
    \textbf{Category} & LV & MA & RF & S-RF & FF-NN & L1-Conc-NN & P-Conc-NN & L1-Pre-Conc-NN  \\
      \hline
      Baby chair & 85.03  & 82.4 & 79.8  & 80.4 & 81.2 & 78.3 & \textbf{75.6} & * \\
      Freezer & 97.5 & 94.6 & 76.7  & 80.2 & \textbf{70.6} & 71.4 & 71.0 & 72.3 \\
      Keyboard &  82.9 & 79.5 & 76.5 & 76.6 & \textbf{68.0} &70.1 & 75.3 & 70.6 \\
      Lawn Mower & 96.0 & 92.3 & \textbf{76.8} & 83.0 & * & 83.0  & 82.1 & * \\
      Scooter & 99.5 &  97.1 & 79.6  & 86.7  & \textbf{76.5}  & 77.8  & 76.9 & 77.9 \\
      SD Cars & 89.1 & 86.7 & 89.5 & 89.9 & 79.0 & \textbf{78.3} & 79.0 & *\\
      Smartphone & 93.6 & 90.1 & 88.2 & 95.6 &  88.3 & 85.4 & 85.3 & \textbf{84.5}  \\
      TVs & 103.5 & 104.0 & 89.9& 97.4 & * &  \textbf{81.4} &  90.8   & *\\
     \hline
\end{tabularx}
}
    \caption{MAPE Results on the market share prediction for an medium-term horizon $h=8$ weeks}
    \label{tab:Result_2}
\end{table}

\begin{table}[]
    \centering
\resizebox{\columnwidth}{!}{
\begin{tabularx}{19cm}{p{3cm}|p{1cm}|X|X|X|X|X|X|X|X|X}
    \textbf{Category} & LV & MA & RF & S-RF & FF-NN & L1-Conc-NN & P-Conc-NN & L1-Pre-Conc-NN  \\
      \hline
      Baby chair & 88.9 & 86.0 & 78.2 & 77.8 & 77.1 & \textbf{76.1}  & 80.0  & 85.9  \\
      Freezer & 97.7 & 94.7 & 73.1  & 76.0 & 72.3 & 70.9 & \textbf{70.0} & 78.1  \\
      Keyboard &  91.5 & 87.0 & 79.5 & 79.1 & \textbf{69.4} & 69.9 & 74.8 & * \\
      Lawn Mower & 99.3  & 96.9 & \textbf{80.2} & 86.4  & 82.7 & 84.5  & 81.5  & 84.5\\
      Scooter & 99.5 & 97.1  & \textbf{79.8} & 87.1  & 80.5 & 85.4 & 84.0 & 83.8 \\
      SD Cars & 95.3 & 91.6  & 86.4 & 89.4 & 82.2  & \textbf{77.2} & 79.4 & * \\
      Smartphone &  97.9 & 95.0  & 91.5  & 100.9  & 89.7 & 85.6 & \textbf{85.3} & 90.5  \\
      TVs & 117.4  & 116.8  & \textbf{97.5}  & 115.1  & * & 101.8  & 110.6 & *\\
     \hline
\end{tabularx}
}
    \caption{MAPE Results on the market share prediction for an long-term horizon $h=12$ weeks}
    \label{tab:Result_3}
\end{table}

\paragraph*{General Remark}
As expected, the error increases with a horizon of prediction $h$. It is expected, as long-term previsions are generally more complex than short-term previsions. However, let us remark that this increasing complexity is not the same for every category. When sales cycles are short, for instance for TVs, the influx of new products makes long-term prediction even harder.

\paragraph*{Comparison with LV, MA} The different Conc-NN models outperform both classical times series estimators LV and MA for almost every horizon and products sets. In particular, traditional time series estimators. It means that Conc-NN can exploit external features. 

\paragraph*{Comparison with RF}
The different Conc-NN models outperform random forests (RF) for almost every product for short-term prediction, but RF becomes better for longer-term horizon. One way to explain this fact is that RF tends to under-predict sales, which favors it for MAPE evaluation. Example of such under-prediction are presented in figure \ref{fig:Sales}

Note that S-RF prediction strongly under-performs RF. Therefore, simply rescaling prediction is not enough to correctly distribute market share.

\paragraph*{Comparison with FF-NN}

For smartphones and TVs, Conc-NN outperforms FF-NN. These categories are the categories that are considered as the most competitive, with a lot of price changes and short sales cycles. It could be proof that our model correctly describes the competition mechanisms in this category. 
	
There may be another explanation, however. Conc-NN also performs well for Scooters, and the three categories (Scooter, TVs, and Smartphones) are also the categories with the most products. They also have good performances on SD Cards, which is also a very competitive category with a few products.

FF-NN outperforms Conc-NN models for keyboards for every horizon.

FF-NN shares a drawback with RF. It under-predicts some products. This may be also due to the L1-Loss minimization, which tends to favor under-predictive models. We also show some examples in Figure \ref{fig:Sales}.

\paragraph*{Pretrained Model L1-Pre-Conc-NN}
Pretrained Concurrent model generally under-performs other Conc-NN. Most of the time, they also under-perform the FF-NN used for pretraining. They also tend to predict 0 a lot. \\
Therefore, pretraining models do not seem to be generally useful. However, they obtain generally good results for Smartphones for all horizons. 

\paragraph*{Poisson Loss VS L1 Loss}
Poisson and L1 Loss leads generally to similar performances. It is hard to see any pattern in the relative performance of both models. Let us just notice that L1 Loss outperforms Poisson for every horizon for TVs, whereas Poisson Loss outperforms L1 for Lawn Mowers, but it could be explained by accident.  

When we observe the prediction in Figure \ref{fig: Sales}, we can notice that L1-Loss-based prediction tends to present higher variation than Poisson-Loss-based prediction. This sensitivity to variation may explain the higher performance of L1-Loss for TVs.

\begin{figure}
    \centering
    \includegraphics[height = 200mm]{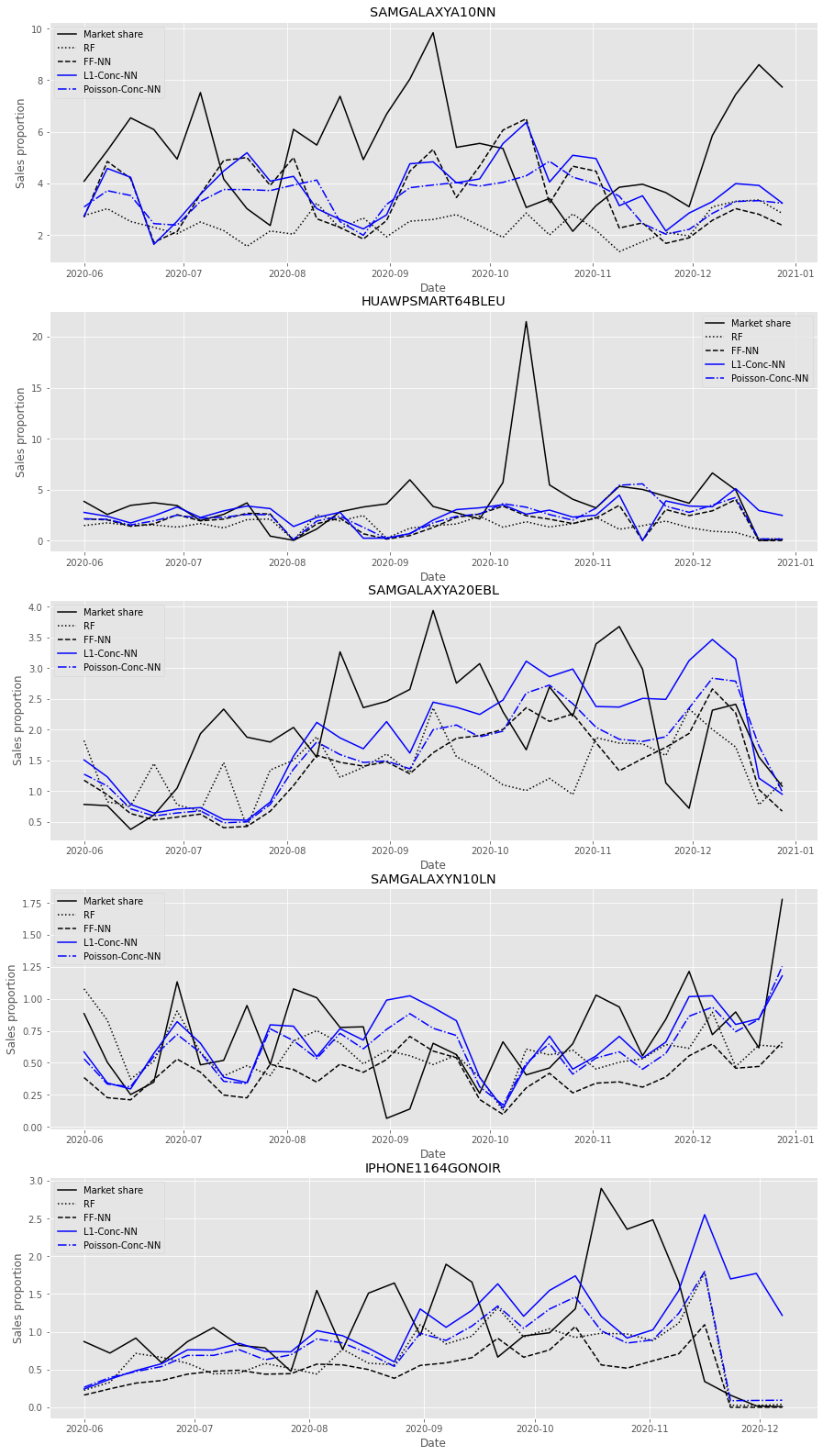}
    \caption{Market prediction for some popular smartphones for an horizon $h=4$ }
    \label{fig:Sales}
\end{figure}

\subsection{Feature Importance}

Our models take into account two external features, the price and margin practiced on the product. To understand how our model treats the covariate, we consider the partial dependance of our models to see if we could explain the variation of underlying weight function $\phi$. To compute the partial dependance graph for a feature $\theta$, we use the following procedures:

\begin{itemize}
    \item We consider the distribution of a given feature $\theta$ in the training set, and we split this distribution into $100$ bins.
    \item For each bin $i$ we compute the average value $\theta_i$ of the feature $\theta$ on the bins.
    \item We compute the average weight on the test set of $\phi(\Tilde{\theta}_{i,t})$, where $\Tilde{\theta}_{i,t}$ is the usual point of data of the test set where the feature $\theta$ has been replaced by $\theta_i$ 
    \item We then plot all the couple (average bin , average weight) .
\end{itemize}

All the partial dependence are computed for the smartphones category, for an horizon $h=4$.

\paragraph*{Partial dependence w.r.t. past proportion}
On Figure \ref{fig:part_dep_prop}, we present the partial dependence of the weight with respect to the last known market share. Logically, it is increasing. When the past sales proportion of the sales were important, it is a strong indication of the competitiveness of a product.

\paragraph*{Partial dependence w.r.t. prices}
On Figure \ref{fig:part_dep_price}, we present the partial dependence of the weight with respect to the price of the smartphone. First, let us note that the overall weight variation is largely smaller than in the last graph . The role of the raw price is less important.

The behavior is somehow odd. Until 500 \euro, increasing the price seems to increase the competitiveness of the product .It may reflect the higher quality of higher price smartphones. When the price is higher, the competitiveness decrease, which is easier to understand. 

\paragraph*{Partial dependence w.r.t. margin}
On Figure \ref{fig:part_dep_margin}, we present the partial dependence of the weight with respect to the margin of the smartphone. Margin is a better indicator than prices, because it takes into account the quality of the product, but suffers from other factors.

Here, when margin is positive, the weight seems to increase with the margin. It is also counter-intuitive, but can also be explained by pricing behavior of the company. When a product receive a great interest, the company can easily increase its margin. The company may also want to push forward products with higher margin.

When margin are negative however, i.e. when products are on sales, we observe a strong increase on its competitiveness.

\begin{figure}
    \centering
    \resizebox{\columnwidth}{!}{
    \includegraphics{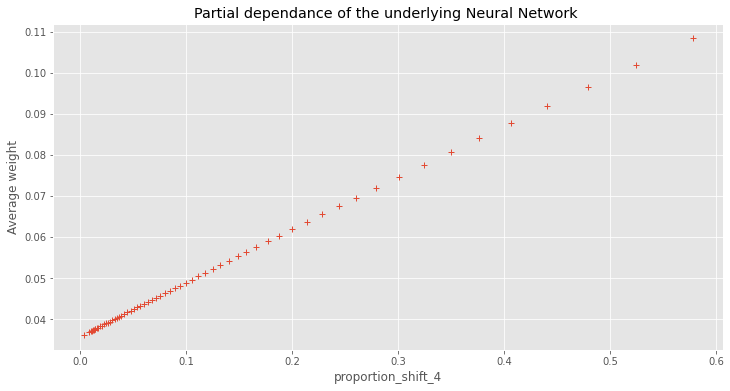} }
    \caption{Partial dependence with respect to the past proportion variables}
    \label{fig:part_dep_prop}
\end{figure}

\begin{figure}
    \centering
    \resizebox{\columnwidth}{!}{
    \includegraphics{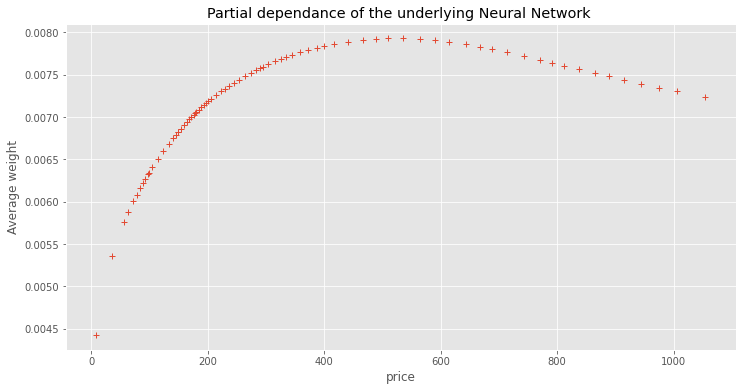} }
    \caption{Partial dependence with respect to the price}
    \label{fig:part_dep_price}
\end{figure}

\begin{figure}
    \centering
    \resizebox{\columnwidth}{!}{
    \includegraphics{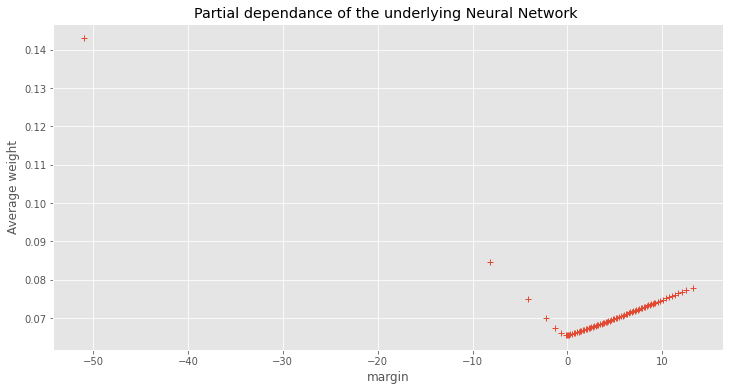} }
    \caption{Partial dependence with respect to the margin percentages}
    \label{fig:part_dep_margin}
\end{figure}

%% file: conclusion.tex
\section{Conclusion}

A model for concurrent time series is proposed in this article. It is based on the creation of unknown "competitiveness" quantity that depends on the characteristic of a product and previous information on its popularity. Under a relatively common condition, we establish a bound on the risk of our model. This bound follows the usual decay in $\mathcal{O}(\sqrt{ \frac{\log(\frac{1}{\delta})}{n}})$ observed in Machine Learning. 

We use this model on real-world data, using a Neural Network to compute the competitiveness using past sales values and other external features. This approach outperforms classical ML and times series estimators, especially for short and medium terms predictions. It also improves classical Neural Network estimators, especially when they are numerous products and when the competition between time series is high. It also partially avoid under-prediction that tends to affect other predictors.

It is possible to explain the weight function. However, the behavior of this function may be counter-intuitive due to the feedback loop.